\documentclass[journal]{IEEEtran}

\usepackage{amssymb,amsmath,amsthm,amsfonts,algorithm,algorithmic,subcaption,graphicx,cite,color,flushend,mysymbol,hyperref,mathrsfs,lipsum,dsfont,multirow}
\IEEEoverridecommandlockouts

\begin{document}

\pagenumbering{gobble}

\title{\LARGE \bf Deep Learning for Robotic Mass Transport Cloaking}

\author{Reza~Khodayi-mehr~and~Michael~M.~Zavlanos%
\thanks{Reza Khodayi-mehr and Michael M. Zavlanos are with the Department of Mechanical Engineering and Materials Science, Duke University, Durham, NC 27708, USA, {\tt\footnotesize \{reza.khodayi.mehr, michael.zavlanos\}@duke.edu}.
This work is supported in part by AFOSR under award \#FA9550-19-1-0169, and by NSF under award CNS-1932011.}
}


\maketitle

\begin{abstract}
We consider the problem of mass transport cloaking using mobile robots. The robots move along a predefined curve that encloses a safe zone and carry sources that collectively counteract a chemical agent released in the environment. The goal is to steer the mass flux around a desired region so that it remains unaffected by the external concentration. We formulate the problem of controlling the robot positions and release rates as a PDE-constrained optimization, where the propagation of the chemical is modeled by the advection-diffusion (AD) PDE. We use a neural network (NN) to approximate the solution of the PDE. Particularly, we propose a novel loss function for the NN that utilizes the variational form of the AD-PDE and allows us to reformulate the planning problem as an unsupervised model-based learning problem. Our loss function is discretization-free and highly parallelizable. Unlike passive cloaking methods that use metamaterials to steer the mass flux, our method is the first to use mobile robots to actively control the concentration levels and create safe zones independent of environmental conditions. We demonstrate the performance of our method in simulations.
\end{abstract}

\IEEEpeerreviewmaketitle

\section{Introduction} \label{sec:intro}
From natural disasters like wildfires \cite{PDTRWF2009MBGA} to environmental pollution \cite{OSIFE2003} and chemical leaks \cite{ILMI1988K,CMSNGE2003L}, many life-threatening processes can mathematically be modeled as distributed parameter systems (DPSs) using transport equations like the advection-diffusion (AD) PDE \cite{IFA1998R}. Creating safe zones in any of these domains, where the temperature or concentration are maintained within the survival limits for humans or animals, is a problem of paramount significance.
The goal is to ensure that these zones remain unaffected by external environmental conditions. In this paper, we propose to use mobile robots to create such safe zones, or cloaks, in mass transport systems modeled by time-dependent AD-PDEs.

Cloaking in transport systems is an active area of research. Examples include mass transport cloaking \cite{FSLT2013GP}, thermal transport cloaking \cite{TC2016SLJNH}, acoustic cloaking \cite{PAC2007CS}, and cloaking for elasticity \cite{CEPETIF2006MBW}. The proposed methods typically rely on metamaterials to passively alter the transport phenomenon around a desired region, decreasing in this way the effect of the external environment. Mathematically, this is achieved by exploiting the invariance of the associated PDEs under curvilinear transformations. However, these methods generally cannot maintain desired environmental conditions in the cloaked zone. 
This is the case in \cite{FSLT2013GP} that creates safe zones in mass transport systems by enclosing those zones in a cloak consisting of cocentric spheres of varying diffusivities. While the concentration inside the safe zone is maintained below the levels of the surrounding environment, it cannot be precisely controlled. This limitation can be alleviated using active cloaking methods, specifically using robots that carry sources that counteract the release of a chemical agent in the environment. By controlling the position of the robots and the release rates of their sources, safe zones of desired concentration levels can be created, even under changing external environmental conditions.
 To the best of our knowledge, this is the first work to consider the problem of active cloaking in mass transport systems using mobile robots.

We model the propagation of the chemical agent in the environment by a time-dependent AD-PDE and assume that the robots move on predefined curves that enclose the desired safe zone. Then, we formulate the proposed active cloaking problem as a PDE-constrained optimization problem whose solution returns collision-free optimal robot trajectories and corresponding source release rates that maintain desired concentration levels in the safe zone. To solve this optimization problem, we propose a new method that employs neural networks (NNs) to approximate the solution of the PDE. 
A major contribution of the proposed method is a novel loss function used to train the NN that employs the variational form of the PDE as opposed to its differential form.
Given this loss function, we formulate the robotic mass transport cloaking (MTC) problem as an unsupervised model-based learning problem that we solve using state-of-the-art stochastic gradient descent algorithms \cite{DLP2017C}.
Common model-based learning approaches in the robotics literature modify a known model using large amounts of labeled data collected before training, e.g., samples of robot trajectories, to account for uncertainties and unmodeled dynamics in the model; see e.g., \cite{SLROAUNSGP2016BMSK}. To the contrary, here we train the NN to approximate the solution of the AD-PDE using unlabeled sample points of the space-time generated during training.
The physics captured by the PDE guide the training of the NN. 

The task of creating safe zones has been considered in the robotics literature, e.g., in containment control \cite{DCCLNPU2012MRM,CCMN2008JFEB} and perimeter patrol \cite{MAPFC2009EAK,MPPAS2008AKK,BCMRPP2009MPAC} problems. In the former, a network of robots are driven to a target destination while contained within a polygon created by the leaders. A possible application is to secure and remove hazardous materials without contaminating the surroundings \cite{CCMN2008JFEB}. In the latter, the objective is to prevent an adversary from penetrating an area of interest by periodically monitoring its boundaries using a team of mobile robots \cite{MPPAS2008AKK}.
A closely related problem is considered in \cite{VFCC2004BCPR} where the objective is to contain a herd of animals within a safe zone.
Unlike these problems, here we are interested in using mobile robots to create safe zones in DPSs where the time-dependent transport phenomenon is governed by PDEs.

The MTC problem is closely related to disturbance control of DPSs \cite{GMCASICDPS2008D} and source identification (SI) \cite{meC1, meC2, meJ1, meC4, meC7}. The objective in the former is to maintain the equilibrium of a DPS against an exogenous disturbance moving in the domain whereas in the latter it is to identify a source function and a corresponding concentration field that matches an observed set of concentrations. Similar to SI, we formulate the MTC problem as a PDE-constrained optimization problem \cite{OPDEC2008HPUU}. The difference is that in the MTC problem the unknown source terms are constrained by the predefined robot curves and the goal is to match the concentration at the safe zone to desired values.
PDE-constrained optimization problems have been studied extensively in the literature \cite{OPDEC2008HPUU}. A major challenge in solving these problems is dealing with the size of the discretized model obtained using numerical methods like the finite element (FE) method \cite{FEM2012H}. This so-called curse of dimensionality is severely exacerbated when the problem is time-dependent.
The only viable solution, currently available, is to use model reduction techniques but these approaches compromise accuracy and stability for speed \cite{SMRMLSS2001ASG}. 

In this paper, we propose an alternative approach to the MTC problem that employs NNs to approximate the solution of the AD-PDE constraint.
Using NNs we can incorporate the AD-PDE constraint in the loss function and obtain a penalty method that circumvents the computational challenges of discretization-based methods to the MTC problem. Specifically, these methods need to reconstruct and exactly enforce new time-dependent AD-PDE constraints at every iteration of the optimization process when the robot positions or release rates change. Unlike discretization-based methods, our learning-based approach is highly parallelizable and can be easily extended to account for different objectives and constraints.
%
%
An overview of different approaches for solving PDEs using NNs can be found in \cite[Ch. 04]{INNMDE2015YYK}.
One group of approaches utilize NNs to memorize the solutions of PDEs. Particularly, they solve the PDE using a numerical method to obtain labeled training data and often utilize convolutional NNs, as powerful image processing tools, to capture the numerical solution in a supervised learning way \cite{SPPDEPANN2017KLY}.
These approaches do not replace numerical methods but rather rely on them and introduce an extra layer of approximation.
%
There also exist methods, called FE-NNs, which represent the governing equations at the element level using artifical neurons \cite{FENNSDE2005RUU}. 
FE-NNs scale with the number of discretization points and are similar in spirit to numerical methods.

Most closely related to the method proposed in this paper are approaches that also directly train a NN to approximate the solution of the PDE in an unsupervised learning process. One of the early works of this kind is \cite{ANNSOPDE1998LLF} that uses the residual of the PDE to define the required loss function. In order to remove the constraints from the training problem, the authors only consider simple domains for which the boundary conditions (BCs) can be manually enforced by a change of variables. 
Although these approaches attain comparable accuracy to numerical methods, they are impractical since, in general, enforcing the BCs might be as difficult as solving the original PDE.
Following a different approach, the work in \cite{SPDEANN2013R} utilizes a constrained back-propagation algorithm to enforce the initial and boundary conditions during training. 
In order to avoid solving a constrained training problem, the authors in \cite{MLPNNNUTM2009SYM} add the constraints corresponding to BCs to the objective as penalty terms. 
Similarly, in \cite{DLASPDE2017SS} the authors focus on the solution of PDEs with high dimensions using a long short-term memory architecture. Note that none of the above methods of solving PDEs using NNs, focus on controlling those PDEs nor do they involve mobile robots in any way.

Compared to the literature discussed above, the contributions of this work can be summarized as follows.
To the best of our knowledge, this work is the first in the robotics literature to consider the active MTC problem. Compared to passive cloaking methods that employ metamaterials to steer the transport phenomenon, our active MTC method can more precisely control the concentration levels in the safe zone.
We propose a new approach to this problem that relies on deep learning to circumvent the computational challenges involved in using discretization-based methods to solve time-dependent PDE-constrained optimization problems.
A major contribution of our proposed learning-based method is a novel loss function that we define to train the NN in an unsupervised model-based way, using the variational form of the PDE.
The advantages of the proposed loss function, compared to existing approaches that use the differential form of the PDE, are two-fold. First, it contains lower order derivatives so that the solution of the PDE can be estimated more accurately.  Note that it becomes progressively more difficult to estimate a function from its higher order derivatives since differential operators are agnostic to translations. Second, it utilizes  the integral (variational) form of the PDE that considers segments of space-time as opposed to single points and imposes fewer smoothness requirements on the solution. Note that variational formulations have been successfully used in the FE literature for a long time \cite{FEM2012H}. 
%
Compared to model-based learning approaches in the robotics literature that rely on supervised learning and large amounts of labeled data, here we propose an unsupervised method that relies on the physics captured by the PDE to guide the training.
Note that although here we consider a specific control problem for the AD-PDE, the proposed principals generalize to other control problems subject to arbitrary dynamics.

The remainder of this paper is organized as follows. In Section \ref{sec:probF} we formulate the MTC problem. Section \ref{sec:sol} is devoted to the solution of the MTC problem using NNs. We present our simulation results in Section \ref{sec:sim} and finally, Section \ref{sec:concl} concludes the paper.

\section{Problem Formulation} \label{sec:probF}

\subsection{Advection-Diffusion PDE}
Let $\Omega \subset \reals^d$ denote a domain of interest where $d$ is its dimension and let $\bbx \in \Omega$ denote a point in this domain and $t \in [0, T]$ denote the time variable.
Furthermore, consider a velocity vector field $\bbu : [0,T] \times \Omega \to \reals^d$ and its corresponding diffusivity field $\kappa: [0,T] \times \Omega \to \reals_+$.
Then, the transport of a quantity of interest $c: [0,T] \times \Omega \to \reals$, e.g., a chemical concentration, in this domain is described by the time-dependent advection-diffusion (AD) PDE \cite{IFA1998R}
\begin{equation} \label{eq:ADPDE}
\dotc = -\nabla \cdot( - \kappa \nabla c + \bbu \, c) + s ,
\end{equation}
where $\dotc = {\partial c}/{\partial t}$ denotes the time derivative of the concentration and $s: [0,T] \times \Omega \to \reals$ is the time-dependent source field that models any chemical reaction or mechanical action that leads to the release or collection of the chemical \cite{RDEA2011K}.

Given an appropriate set of initial conditions (ICs) and boundary conditions (BCs), it can be shown that the AD-PDE \eqref{eq:ADPDE} is well-posed and has a unique solution \cite{IFA1998R}. In this paper, we use the following IC and BCs:
\begin{subequations} \label{eq:IBCs}
\begin{align}
&c(0,\bbx) = g_0(\bbx) \ \ \, \for \bbx \in \Omega , \label{eq:IC} \\
&c(t,\bbx) = g_i(t,\bbx) \ \for \bbx \in \Gamma_i ,  \label{eq:BCs}
\end{align}
\end{subequations}
where $\Gamma_i$ for $i \in \set{1, \dots, n_b}$ denote the boundary segments of $\Omega$ and $g_0: \Omega \to \reals$ and $g_i: [0,T] \times \Gamma_i \to \reals$ possess appropriate regularity conditions \cite{IFA1998R}.
Equation \eqref{eq:IC} describes the state of the concentration field before the release of the chemical agent whereas equation \eqref{eq:BCs} prescribes the concentration value along the boundary segment $\Gamma_i$ as a time-varying function.

\subsection{Mass Transport Cloaking using Mobile Robots}
Given the AD-PDE \eqref{eq:ADPDE} and the domain of interest $\Omega$, consider a team of robots that carry sources that collectively counteract the release of a chemical. These mobile robots control the source term $s(t,\bbx)$ in the AD-PDE \eqref{eq:ADPDE}.
The goal of the robots is to create a safe zone in the domain where the concentration is controlled to remain at a certain level. This can be achieved by formulating a planning problem that controls the robot positions in the domain as well as the release rates of their sources.
Figure \ref{fig:setup} shows a typical scenario where we assume the chemical agent is released outside $\Omega$ and its effect is modeled by the BCs.
\footnote{Inclusion of the chemical release in the domain is equivalent to adding an extra source term in \eqref{eq:ADPDE}.}
\begin{figure}[t!]
  \centering
    \includegraphics[width=0.35\textwidth]{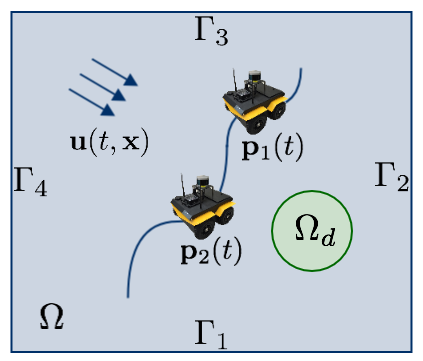}
  \caption{A typical setup for the mass transport cloaking problem. The robots release or collect a chemical agent in an optimal rate along their paths $\bbp_j(t)$, in response to changing environmental conditions, so that the concentration at the safe zone $\Omega_d$ is maintained at a desired value.} \label{fig:setup}
\end{figure}
In this figure, $\Omega_d \subset \Omega$ denotes the desired safe zone. The goal is to maintain the release of the chemical invisible to an observer inside $\Omega_d$. Specifically, let $N$ denote the number of robots. Then, the path of robot $j$ is denoted by $\bbp_j: [0,T] \to \Omega \setminus \Omega_d$ for $j \in \set{1, \dots, N}$.
We make the following assumptions.
\begin{assumption} [PDE Input-Data] \label{assum:inputData}
The geometry of the domain $\Omega$ and its boundaries $\Gamma_i$, the velocity $\bbu(t,\bbx)$ and corresponding diffusivity $\kappa(t,\bbx)$ fields, and the IC $g_0(\bbx)$ and BCs $g_i(t,\bbx)$  for $i \in \set{1, \dots, n_b}$ are known.
\end{assumption}

The assumption that the domain $\Omega$ is known is a reasonable one. Given the knowledge of the domain, the velocity and diffusivity fields can be estimated using computational fluid dynamics or from measurements. Moreover, knowledge of the release of a chemical agent captured by the BCs, might be known for a specific task or can be measured.

\begin{assumption} [Robot Paths] \label{assum:path}
The robot paths are constrained on a curve $\bbgamma: [0,1] \to \Omega \setminus \Omega_d$ that is simple, differentiable, and does not intersect with the safe zone $\Omega_d$ or other obstacles in the domain.
\end{assumption}

The curve $\bbgamma$ often encloses $\Omega_d$ and in many robotic applications corresponds to the perimeter of the safe zone \cite{DCCLNPU2012MRM,MAPFC2009EAK}.
Given this curve, we can define the path of robot $j$ as
\begin{equation} \label{eq:path}
\bbp_j(t) = \bbgamma(\xi_j(t)) ,
\end{equation}
where $j \in \set{1, \dots, N}$ and $\xi_j: [0,T] \to [0,1]$ is a parameterization that maps the time interval $[0,T]$ to $[0,1]$. Note that $\xi_j(t)$ is not necessarily monotone, i.e., the robot can move back and forth along the curve $\bbgamma$.

The release of the source, carried by robot $j$, is limited to the location of the robot at any given time. In order to capture this local effect, we model the source term using a Gaussian function centered at the robot. Particularly, let $l_s > 0$ denote the length-scale of the Gaussian function and $a_j: [0,T] \to \reals$ the release rate of the source. Then, the source term $s(t, \bbx)$ in \eqref{eq:ADPDE} is given by
\begin{equation} \label{eq:source}
s(t, \bbx) = \sum\nolimits_{j=1}^N a_j(t) \exp \left( - \frac{ \norm{\bbp_j(t) - \bbx}^2 }{2 \, l_s^2} \right) ,
\end{equation}
where $\bbp_j(t)$ denotes the location of the robot $j$ at time $t$.

\begin{prob} [Mass Transport Cloaking] \label{prob:MTC}
Let $c_d(t, \bbx)$ denote the desired concentration level in the safe zone $\Omega_d$. Find the optimal paths $\bbp_j^*(t)$ and release rates $a_j^*(t)$ to maintain concentration $c_d(t, \bbx)$ inside $\Omega_d$.
\end{prob}

Problem \ref{prob:MTC} can be formulated as the following PDE-constrained optimization problem:
\begin{align} \label{eq:mission}
&\min_{a_j(t), \bbp_j(t)} \norm{c(t,\bbx) - c_d(t, \bbx)}_{\Omega_d}^2, \\
&\ \ \ \ \st \ccalL c = s(t,\bbx) \ \ \ \ \ \ \for \bbx \in \Omega , \nonumber \\
& \ \ \ \ \ \ \ \ \, c(0,\bbx) = g_0(\bbx) \ \ \, \for \bbx \in \Omega , \nonumber \\
& \ \ \ \ \ \ \ \ \, c(t,\bbx) = g_i(t,\bbx) \ \for \bbx \in \Gamma_i , \nonumber
\end{align}
for $i \in \set{1, \dots, n_b} \and j \in \set{1, \dots, N}$, where $t \in [0,T]$ in the relevant constraints, $\ccalL c = \dotc + \nabla \cdot( - \kappa \nabla c + \bbu \, c)$ is the differential operator corresponding to the AD-PDE \eqref{eq:ADPDE}, and the objective of the optimization is defined as
\begin{equation} \label{eq:obj}
\norm{c(t,\bbx) - c_d(t, \bbx)}_{\Omega_d}^2 = \int_0^T \int_{\Omega_d} \abs{ c(t,\bbx) - c_d(t, \bbx) }^2 \, d\bbx \, dt.
\end{equation}
Note that in equation \eqref{eq:obj}, the spatial integration is performed only over the safe zone $\Omega_d$.


\section{Solution using Neural Networks} \label{sec:sol}
In this section we first discuss the loss function that we use to train a neural network (NN) to approximate the solution of the time-dependent AD-PDE \eqref{eq:ADPDE} for a given set of PDE input-data. Then, given this loss function, we address the mass transport cloaking problem \eqref{eq:mission}. The proposed loss function uses the variational formulation of the AD-PDE while our solution to problem \eqref{eq:mission} relies on adding this loss function, that captures the constraints in problem \eqref{eq:mission}, to objective \eqref{eq:obj} in the form of a penalty term, to formulate an unconstrained training problem that can be solved using stochastic gradient descent methods.

\subsection{Neural Network Approximation of the AD-PDE} \label{sec:NN}
Let $\bbtheta \in \reals^n$ denote the weights and biases of the NN, a total of $n$ trainable parameters. Then the solution of the AD-PDE can be approximated by the nonlinear function $f(t,\bbx; \bbtheta)$, where $f: [0,T] \times \Omega \to \reals$ maps the inputs $t$ and $\bbx$ of the NN to its scalar output.
The objective of training is to learn the parameters $\bbtheta$ of the network so that $f(\cdot)$ approximates the solution of the AD-PDE \eqref{eq:ADPDE} as well as possible. To capture this, we need to define a loss function $\ell: \reals^n \to \reals_+$ that reflects how well the function $f(t,\bbx; \bbtheta)$ approximates the solution of AD-PDE \eqref{eq:ADPDE}. A typical approach is to consider the residual of the PDE when $f(t,\bbx; \bbtheta)$ is substituted into the differential form \eqref{eq:ADPDE}. This approach has two problems. First, for the AD-PDE \eqref{eq:ADPDE} it requires evaluation of second order derivatives of $f(t,\bbx; \bbtheta)$. Estimating a function from its higher order derivatives is inefficient since differentiation only retains slope information and is agnostic to translation. Second, training the differential form of the PDE amounts to learning a complicated field by only considering the PDE-residual at training points, i.e., a measure-zero set, 
which is ineffective in capturing the physics of the transport phenomenon.

We follow a different approach here that addresses both of these issues. Particularly, we rely on the variational form of the PDE. Let $v: [0,T] \times \Omega \to \reals$ be an arbitrary compactly supported test function. Multiplying \eqref{eq:ADPDE} by $v(t, \bbx)$ and integrating over the spatial and temporal coordinates we get
\begin{align*}
\int_0^T \int_{\Omega} v \left[ \dotc + \nabla \cdot( - \kappa \nabla c + \bbu \, c) - s \right] d\bbx \, dt = 0, \ \ \forall v.
\end{align*}
Performing integration by parts we have
\begin{align*}
 \int_0^T \int_{\Omega} \dotc \, v d\bbx \, dt &= \int_{\Omega} \left[ c(T) v(T) - c(0) v(0) \right] d\bbx \\
& - \int_0^T \int_{\Omega} c \, \dot{v} d\bbx \, dt = - \int_0^T \int_{\Omega} c \, \dot{v} d\bbx \, dt .
\end{align*}
Note that $v(0) = v(T) = 0$ since $v(t, \bbx)$ is compactly supported over the temporal coordinate. Similarly,
$$ \int_0^T \int_{\Omega} v \, \nabla \cdot \kappa \nabla c \, d\bbx \, dt = 
- \int_0^T \int_{\Omega} \nabla v \cdot \kappa \nabla c \, d\bbx \, dt , $$
where again the boundary terms vanish since $v(t, \bbx)$ is compactly supported.
Finally, note that for velocities far below the speed of sound the incompressibility assumption holds and $\nabla \cdot \bbu = 0$. Putting together all of these pieces, we get the variational form of the time-dependent AD-PDE as 
\begin{equation} \label{eq:varPDE}
l(c,v) = \int_0^T \int_{\Omega} \left[ \nabla c \cdot \left( \kappa \nabla v + \bbu \, v \right) - c \, \dot{v} - s \, v \right] \, d\bbx \, dt = 0.
\end{equation}

The variational form \eqref{eq:varPDE} only requires the first-order spatial derivative and also an integration over a non-zero measure set, i.e., the support of the test function, as opposed to a single point used in the differential form. The test function acts as a weight on the PDE residual and the idea is that if \eqref{eq:varPDE} holds for a reasonable number of test functions $v(t, \bbx)$ with their compact supports located at different regions in space-time $[0,T] \times \Omega$, the function $f(t, \bbx)$ has to satisfy the PDE.
A very important feature of the test function $v(t, \bbx)$ is that it is compactly supported. This allows local treatment of the PDE as opposed to considering the whole space-time at once and is the basis of the finite element method \cite{FEM2012H}. More details regarding our proposed loss function and its advantages can be found in \cite{meJ5}.

Given the variational form \eqref{eq:varPDE}, we can now define the desired loss function to approximate the solution of the AD-PDE \eqref{eq:ADPDE} for a given set of input data. Consider a set of $n_v$ test functions $v_k(t, \bbx)$ sampling the space-time $[0,T] \times \Omega$, a set of $n_0$ points $\bbx_k \in \Omega$ corresponding to the IC, and sets of $n_{b,i}$ points $(t_k, \bbx_k) \in [0,T] \times \Gamma_i$ for the enforcement of the BCs. Then, we define the loss function $\ell: \reals^n \to \reals_+$ as
\begin{align} \label{eq:loss}
\ell(\bbtheta) &= w_1 \sum\nolimits_{k=1}^{n_v} \abs{ l(f, v_k) }^2 
 + \frac{w_2}{n_0} \sum\nolimits_{k=1}^{n_0} \abs{f(0,\bbx_k) - g_0(\bbx_k)}^2  \nonumber \\
& + \frac{w_3}{\bbarn_b} \sum\nolimits_{i=1}^{n_b} \sum\nolimits_{k=1}^{n_{b,i}} \abs{f(t_k, \bbx_k) - g_i(t_k, \bbx_k)}^2 ,
\end{align}
where $l(f, v_k)$ is given by \eqref{eq:varPDE}, $\bbw \in \reals_+^3$ stores the penalty weights corresponding to each term, and $\bbarn_b = \sum_{i=1}^{n_b} n_{b,i}$ is the total number of training points for the BCs. Note that in the first term in \eqref{eq:loss}, the integration is limited to the support of $v(t, \bbx)$ which is computationally very advantageous.
Furthermore, this loss function is lower-bounded by zero. Since this bound is attainable for the exact solution of the AD-PDE \eqref{eq:ADPDE}, the value of the loss function is an indicator of how well the NN approximates the solution of the PDE.

Training using the loss function \eqref{eq:loss} is an instance of unsupervised learning since the solution is not learned from labeled data. Instead, the training data here are unlabeled samples of space-time and the physics captured by the AD-PDE \eqref{eq:ADPDE} guides learning of the parameters $\bbtheta$. In that sense, our approach is a model-based method as opposed to a merely statistical method that automatically extracts features and cannot be easily interpreted.
Moreover, unlike common robotic model-based learning approaches that modify a known model to account for the labeled data, obtained through numerous experiments, our approach proposes a novel use of NNs to learn the model from scratch without relying on labeled data.

\begin{rem} [Convergence]
By the universal approximation theorem, for a large enough number of trainable parameters $n$, the NN can approximate any smooth function \cite{MFNUA1989HSW}. Given that the AD-PDE \eqref{eq:ADPDE} has a unique smooth solution given an appropriate set of input data, the NN should converge to this solution when $n \to \infty$; see \cite{DLASPDE2017SS} for details.
\end{rem}

\begin{rem} [Parallelization]
Referring to \eqref{eq:loss}, parallelization of the training process is trivial. We can choose the number of samples $n_v, n_0, \and n_{b,i}$ in accordance to the available computational resources and decompose and assign the summations to different processing units. The locations of the training points are also arbitrary and can be selected through random drawing or from a fixed grid over space-time.
\end{rem}

\subsection{Mass Transport Cloaking using Neural Networks} \label{sec:NNplan}
So far we have described how to approximate the solution of the AD-PDE \eqref{eq:ADPDE} using a NN for a given set of input data. Next, we discuss how to use this approximation to solve the mass transport cloaking (MTC) problem \eqref{eq:mission}. The idea is to add the NN loss function that captures the constraints of the problem \eqref{eq:mission} to the objective in the form of a penalty term, and obtain an unconstrained optimization problem that can be solved using state-of-the-art stochastic gradient descent algorithms. Such algorithms have been very effective in the deep learning literature \cite{DLP2017C}.

Specifically, to solve the PDE-constrained optimization problem \eqref{eq:mission}, we parameterize the release rate $a_j(t)$ by a polynomial 
$a_j(t) = \sum\nolimits_{k=0}^{n_{\alpha}} \alpha_{jk} t^k,$
%
%
where $\bbalpha_j \in \reals^{n_{\alpha}}$ is the vector of coefficients corresponding to robot $j$.
We also parameterize $\xi_j(t)$ in the definition of the robot paths \eqref{eq:path} by the composition of another polynomial $b_j(t) = \sum_{k=0}^{n_{\beta}} \beta_{jk} t^k$ with parameters $\bbbeta_j \in \reals^{n_{\beta}}$ and the sigmoid function $\sigma: \reals \to [0,1]$ defined as
$ \sigma(b) = {1}/{1 + \exp(-b)} .$
Thus,
\begin{equation} \label{eq:param}
\xi_j(t) = \sigma \left( \sum\nolimits_{k=0}^{n_{\beta}} \beta_{jk} t^k \right) .
\end{equation}
Furthermore, we use a set of $n_d$ points $(t_k, \bbx_k) \in [0,T] \times \Omega_d$ for the discrete approximation of the integral \eqref{eq:obj}.
Noting that the loss function \eqref{eq:loss} captures the constraints in \eqref{eq:mission}, to solve the MTC problem we optimize the following objective function
\begin{equation} \label{eq:obj2}
J(\bbalpha_j, \bbbeta_j, \bbtheta) = \frac{1}{n_d} \sum\nolimits_{k=1}^{n_d} \abs{f(t_k, \bbx_k) - c_d(t_k, \bbx_k)}^2 + \ell(\bbtheta) .
\end{equation}
Optimizing \eqref{eq:obj2} provides the solution of the AD-PDE \eqref{eq:ADPDE} for the time-dependent source term \eqref{eq:source}, with intensities specified by $\bbalpha^*_j$ and paths specified by $\bbbeta^*_j$, that maintains the concentration at the desired level $c_d$ across the safe zone $\Omega_d$.
Note that for simplicity we used the same length scale $l_s$ and number of parameters $n_{\alpha}$ and $n_{\beta}$ for all robots in the source term \eqref{eq:source}; generalization is trivial.


\begin{prop} [Bounded Velocity]
Robots following the paths \eqref{eq:path} with parameterizations \eqref{eq:param} have bounded velocities.
\end{prop}
\begin{proof}
Noting that $\bbp(t) = \bbgamma(\xi(t))$, for the speed of the robot we have $\norm{\dot{\bbp}(t)} = |\dot{\xi}(t)| \norm{\nabla_{\xi} \bbgamma}$, where we drop the subscript $j$ for simplicity. Since $\bbgamma(\xi)$ is differentiable and defined over the compact set $[0,1]$, it is bounded. Thus, we need to show that $|\dot{\xi}(t)|$ is bounded.
To see this, note that 
$$ \dot{\xi}(t) = \dot{b} \, \exp(-b) \, \sigma^2(b) = \frac{\dot{b} \, \exp(-b)}{[1 + \exp(-b)]^2} , $$
where $b(t)$ is the polynomial that parameterizes the path of the robot. Then, since
$$ \lim_{t \to \infty} \frac{ \dot{b} }{b} = \lim_{t \to \infty} \frac{n_{\beta} \, \beta_{n_{\beta}} t^{n_{\beta}-1} + o(t^{n_{\beta}-1})}{ \beta_{n_{\beta}} t^{n_{\beta}} + o(t^{n_{\beta}}) } = 0 $$
and $\lim_{b \to +\infty} b/\exp(b) = 0$, we have
$\lim_{t \to \infty} \dot{\xi}(t) = 0 .$
Since $\dot{\xi}(t)$ is also continuously differentiable, it has to be bounded. Intuitively, when $b \to \pm \infty$, the sigmoid function saturates and the robot approaches one end of the curve $\bbgamma$.
\end{proof}

\section{Numerical Experiments} \label{sec:sim}
In this section we present simulation results that illustrate our approach to solving the MTC problem \eqref{eq:mission}.
The domain is depicted in Figure \ref{fig:setup} where we select $\Omega = [-1,1] \times [-1,1]$ and $T = 1$. Before the release of the chemical agent, the system is at rest with zero concentration across the domain, i.e., $g_0(\bbx) = 0$ in \eqref{eq:IC}. At time $t=0$ the release occurs causing the concentration to rise linearly at the corner $\bbx_c = [-1,1]$ of the boundaries $\Gamma_3 \and \Gamma_4$ as
$ c(t, \bbx_c) = \min \set{1, {2t}/{T} } \for t \in [0,T] .$
Note that the concentration $c(t, \bbx_c)$ reaches the maximum value of $1$ at $t=0.5\, T$ and stays constant afterwards. For the boundary $\Gamma_3$, the concentration decreases linearly to zero when moving from $\bbx_c$ toward the corner between $\Gamma_3 \and \Gamma_2$, i.e., in \eqref{eq:BCs} we have
$ g_3(t, \bbx) = - 0.5 \, { c(t, \bbx_c) } \, (x_1 - 1) \for \bbx \in \Gamma_3 , $
see Figure \ref{fig:BC}.
\begin{figure}[t!]
  \centering
    \includegraphics[width=0.42\textwidth]{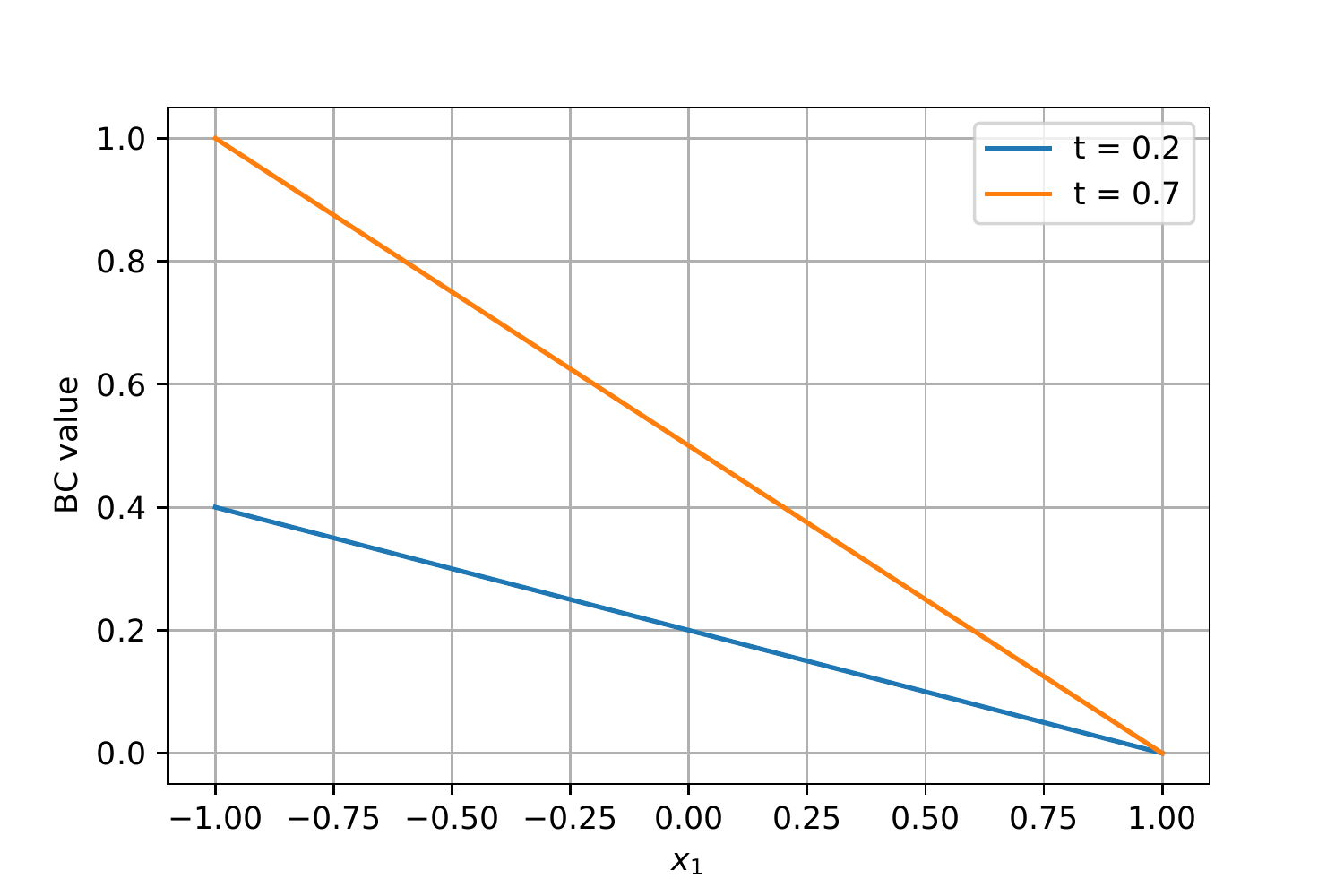}
  \caption{BC value $g_3(t, \bbx)$ for $\Gamma_3$ as a function of spatial coordinate for two different time instances $t=0.2 \and t = 0.7$ before and after $0.5 \, T$.} \label{fig:BC}
\end{figure}
A similar expression is used for $g_4(t, \bbx)$ along $\Gamma_4$. The opposite boundaries $\Gamma_1 \and \Gamma_2$ are set to zero for all time, i.e., $g_1(t, \bbx) = g_2(t, \bbx) = 0$ in \eqref{eq:BCs}.
Finally, we set the diffusivity field to a constant value of $\kappa(t, \bbx) = 1$ and define the velocity vector field for $t \in [0, 0.5T]$ as
\begin{equation} \label{eq:velF}
\bbu(t, \bbx) = \left[ \cos(- {\pi t}/{T}), \, \sin(- {\pi t}/{T}) \right] 
\end{equation}
and $\bbu(t, \bbx) = [0, -1] \for t > 0.5T$, i.e., it has a constant magnitude $\norm{\bbu(t, \bbx)} = 1$ but its direction rotates $90^o$ from $[1,0]$ to $[0, -1]$ within $t \in [0, 0.5T]$ and then stays constant.
Note that variations in the BCs and velocity field stop at $t=0.5T$ allowing $c(t,\bbx)$ to move toward steady-state.

We define the safe zone $\Omega_d$ as a square centered at $\bbx_d = [-0.4, 0.3]$ with side length of $0.2$ and set $c_d = 0$, i.e., we wish to maintain zero concentration at the safe zone as it was before the release of the chemical.
In the following simulations, we use $N=2$ mobile robots that collaboratively cloak the safe zone $\Omega_d$.
Referring to Assumption \ref{assum:path}, we constrain the robots to move along a circular path with radius $r = 0.4$ around $\bbx_d$. Particularly, the admissible path for robot $j$ is defined as
$$ \bbp_j(t) = \bbgamma(\xi_j) = \bbx_d + r \left[ \cos(2 \pi \xi_j), \, \sin(2 \pi \xi_j) \right] , $$
where $\xi_j \in [0,1]$ is the parameterization corresponding to robot $j$.
In order to describe the source term \eqref{eq:source}, we use $n_{\alpha} = 5$ parameters for the release rates of the robots and $n_{\beta} = 4$ parameters for their paths, where we fix the intercepts $\beta_{10} = 0 \and \beta_{20} = \ln(3)$ of the polynomials in \eqref{eq:param} to specify the initial angular positions of the robots on the curve $\bbgamma$. These values correspond to $\vartheta_1(0) = 180^o \and \vartheta_2(0) = 270^o$ where $\vartheta_j(t) = 2 \pi \xi_j(t)$.
We also set the characteristic length of the Gaussian source \eqref{eq:source} to $l_s = 0.04$; see Section \ref{sec:NNplan} for details.
Then, the robots optimally adjust their velocities along this curve and their release rates in response to the change in the BCs $g_3(t,\bbx) \and g_4(t,\bbx)$ and the velocity field $\bbu(t, \bbx)$ to maintain the concentration $c_d$ in $\Omega_d$.
In order to prevent collision among the robots, we add the following penalty term (with an appropriate weight) to the training objective \eqref{eq:obj2}
\begin{equation*} 
J_c(t) = \sum\nolimits_{j=1}^N \sum\nolimits_{k=j+1}^N \exp \left[ - \left( \frac{2 \pi [ \xi_j(t) - \xi_k(t) ] }{ \vartheta_0 } \right)^2 \right] ,
\end{equation*}
measuring the mutual angular distance of the robots, where $\vartheta_0$ determines the angle margin among robots and is set to $\vartheta_0 = 4^o$ in the following results.

We utilize the \textsc{TensorFlow} software to solve the desired MTC problem \eqref{eq:mission}; see \cite{tensorflow2016ABCCD}.
The test functions $v_i(t, \bbx)$, introduced in Section \ref{sec:NN}, are selected to be trilinear finite element (FE) basis functions for $3$D hexagonal elements that are centered at arbitrary points in space-time and, the numerical integration in \eqref{eq:varPDE} is performed using a two-point Gauss-Legendre quadrature at each dimension; see \cite{meJ5} for more details.
We use a multi-layer perceptron (MLP) NN architecture consisting of three dense hidden layers with $10, 20, \and 30$ neurons, respectively and sigmoid activation functions. In order to train the NN, we use $50$ samples of the temporal coordinate as well as $40 \times 40$ samples of the spatial domain resulting in $n_v = 50 \times 40 \times 40$ training points in the domain interior, $n_0 = 40 \times 40$ training points for the IC, $n_{b,i} = 50 \times 40$ training points for each BC, and finally $n_d = 50 \times 16$ training points for the discrete approximation of the objective \eqref{eq:obj} at the safe zone $\Omega_d$.
We set $\bbw = [1, 1, 10]$. Then, the unsupervised training for the MTC problem \eqref{eq:mission} is performed using the \texttt{AdamOptimizer}; see \cite{DLP2017C} for details.
To validate the resulting optimal control inputs for the robots, we solve the AD-PDE \eqref{eq:ADPDE} for these inputs using an in-house FE solver that we have developed using an explicit forward Euler scheme and \textsc{DiffPack} libraries \cite[Ch. 3.10]{CPDE2013L}. We use a mesh with $50 \times 40 \times 40$ nodes corresponding to the number of training points above.

Figure \ref{fig:conc} shows the evolution of the concentration at $\bbx_d$ obtained from the NN and FE solutions along with maximum and minimum (worst case) concentration values across the safe zone according to the FE solver.
\begin{figure}[t!]
        \centering
	\includegraphics[width=0.43\textwidth]{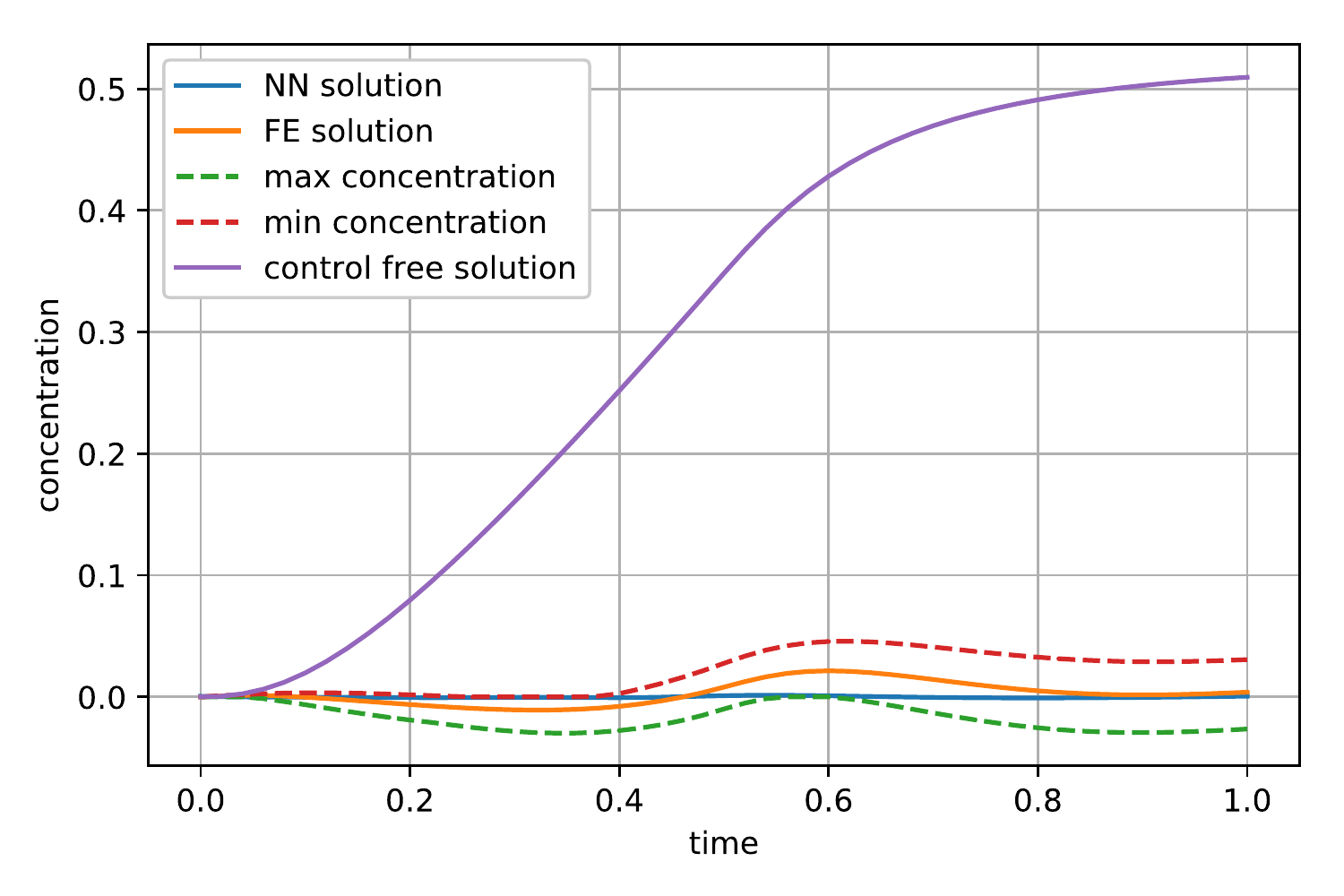}
	\caption{Concentration level at $\bbx_d$ obtained from the NN and FE solutions as a function of time. The dashed lines delineate the maximum deviations from the desired value $c_d=0$ across the safe zone $\Omega_d$ according to the FE solution.} \label{fig:conc}
\end{figure}
It can be seen that the concentration stays close to the desired value $c_d=0$. The highest concentration levels in $\Omega_d$ are more than an order of magnitude smaller than the maximum concentration $c(T, \bbx_c)=1$, meaning that external concentration levels are effectively invisible in the safe zone $\Omega_d$.
Next, Figure \ref{fig:sourceParam} shows the optimal release rates $a_j^*(t)$ and angular positions $\vartheta_j^*(t) = 2 \pi \xi^*_j(t)$ as functions of time.
\begin{figure}[t!]
        \centering
    \begin{subfigure}[b]{0.43\textwidth}
                \includegraphics[width=\textwidth]{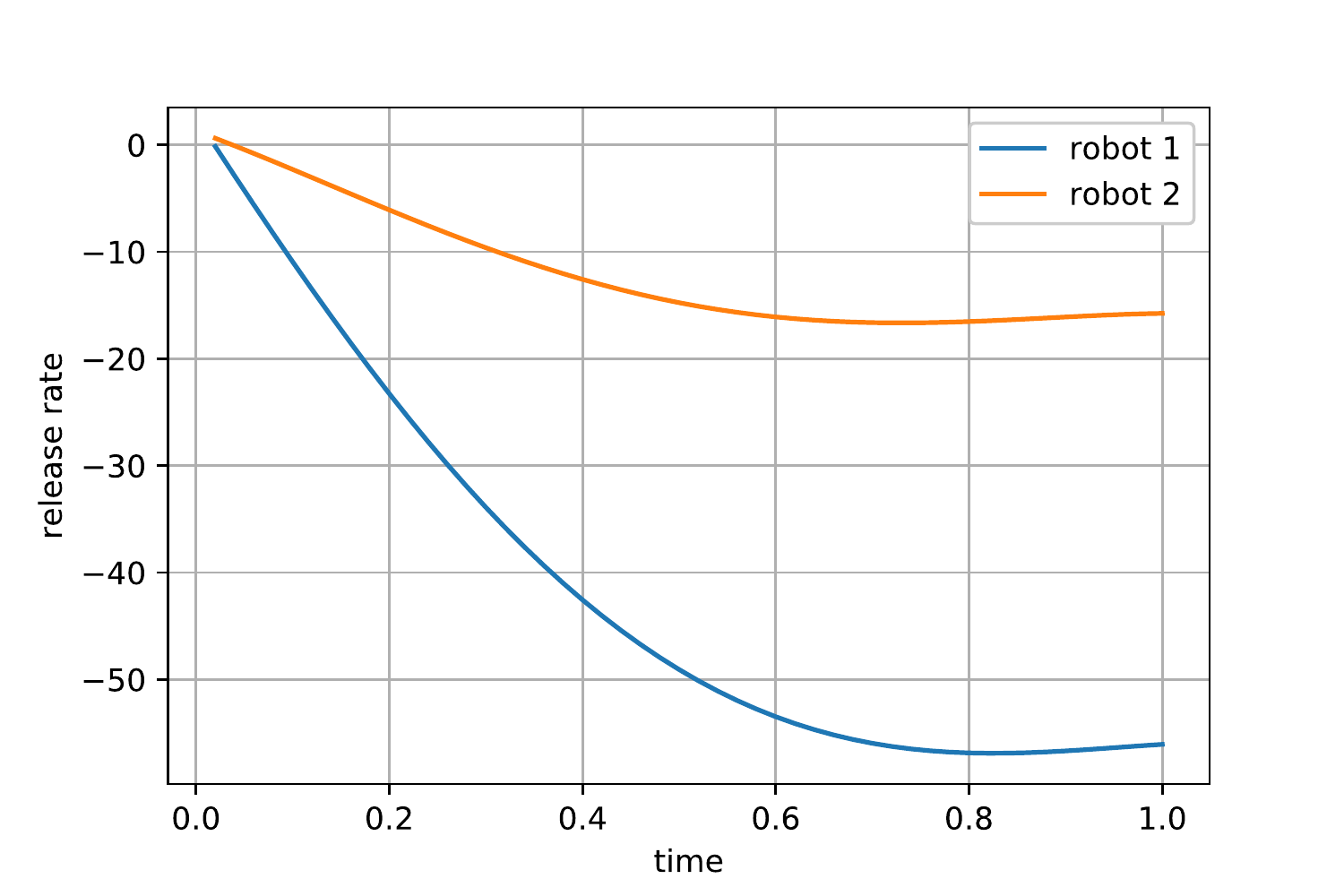}
    \caption{optimal release rates $a_j^*(t)$} \label{fig:intensity}
\end{subfigure}
    \quad
    \begin{subfigure}[b]{0.43\textwidth}
                \includegraphics[width=\textwidth]{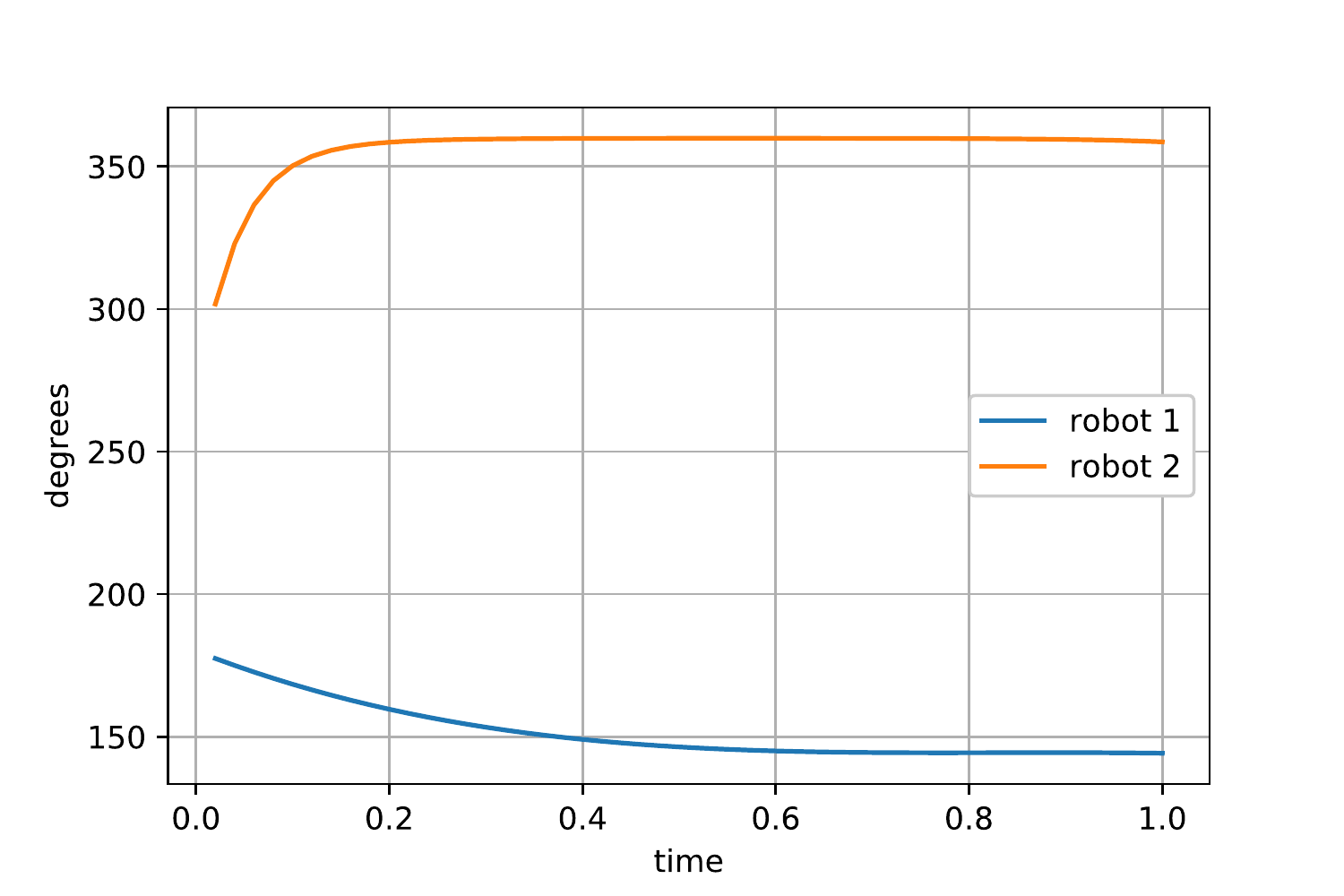}
    \caption{optimal angular positions $\vartheta_j^*(t)$} \label{fig:angle}
\end{subfigure}
	\caption{Optimal source parameters as a function of time to maintain the concentration levels at $\Omega_d$ at the specified level $c_d$.} \label{fig:sourceParam}
\end{figure}
Note that the release rates in Figure \ref{fig:intensity} decrease to negative values in response to the growth of the concentration at the corner $\bbx_c$ of the boundary $\Gamma_3$. This corresponds to robots acting as sinks and collecting the extra chemical to maintain desired concentration levels in $\Omega_d$. 
Referring to Figure \ref{fig:angle}, the angular positions are adjusted in response to the rotation of the velocity field $\bbu(t)$ so that the safe zone $\Omega_d$ is positioned between the robots. Not that since robot 2 is positioned downstream of the safe zone, for cloaking it relies on diffusive transport, which is as strong as advection here.
Finally, Figure \ref{fig:solPDE} shows snapshots of the concentration field, given by the NN approximation of the solution of the AD-PDE \eqref{eq:ADPDE}, at different time instances.
\begin{figure*}
	\centering
	\begin{subfigure}[b]{0.31\textwidth}
		\includegraphics[width=\textwidth]{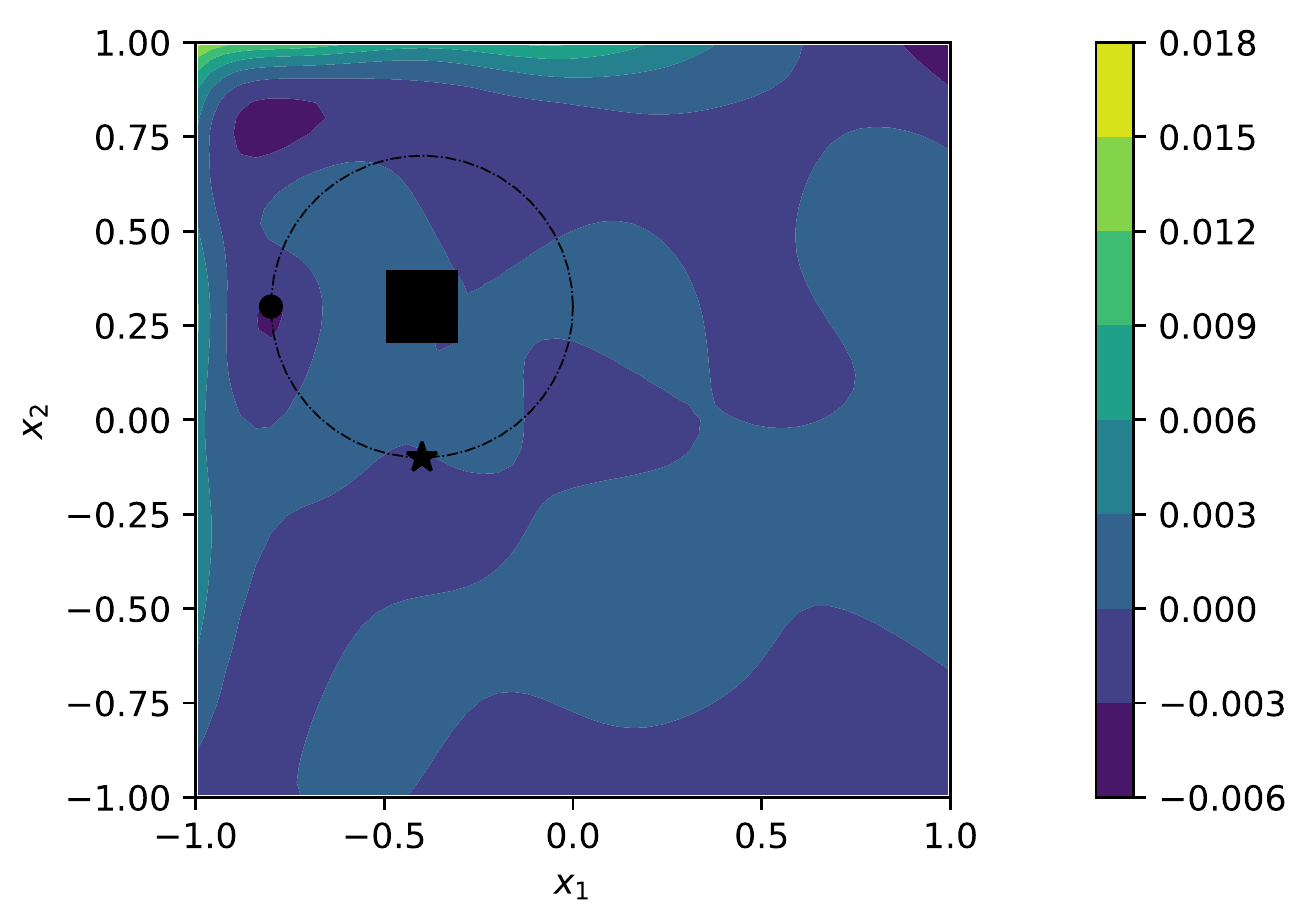}
	\caption{$t=0$}
	\end{subfigure}
	\begin{subfigure}[b]{0.31\textwidth}
		\includegraphics[width=\textwidth]{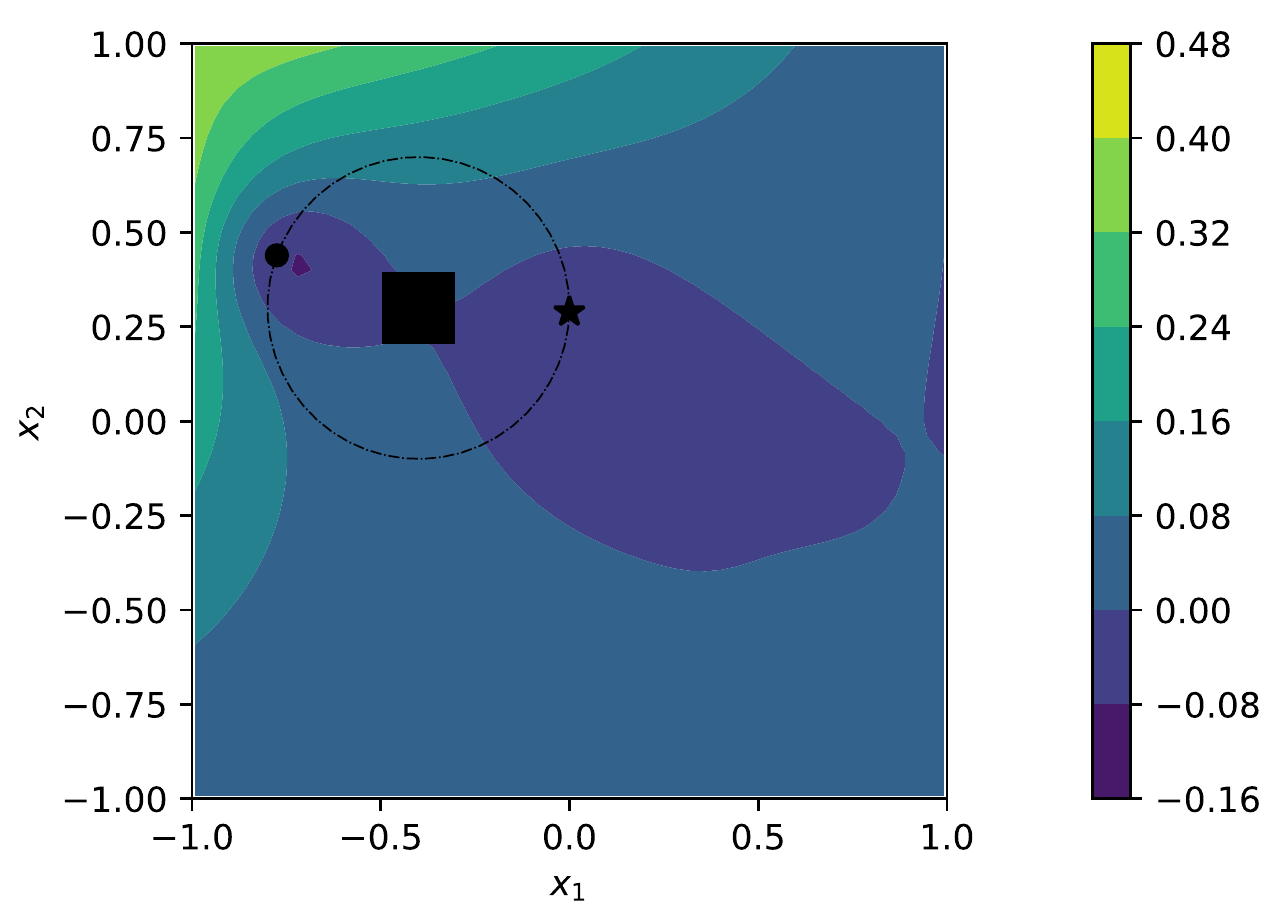}
	\caption{$t=0.2$}
\end{subfigure}
	\begin{subfigure}[b]{0.31\textwidth}
		\includegraphics[width=\textwidth]{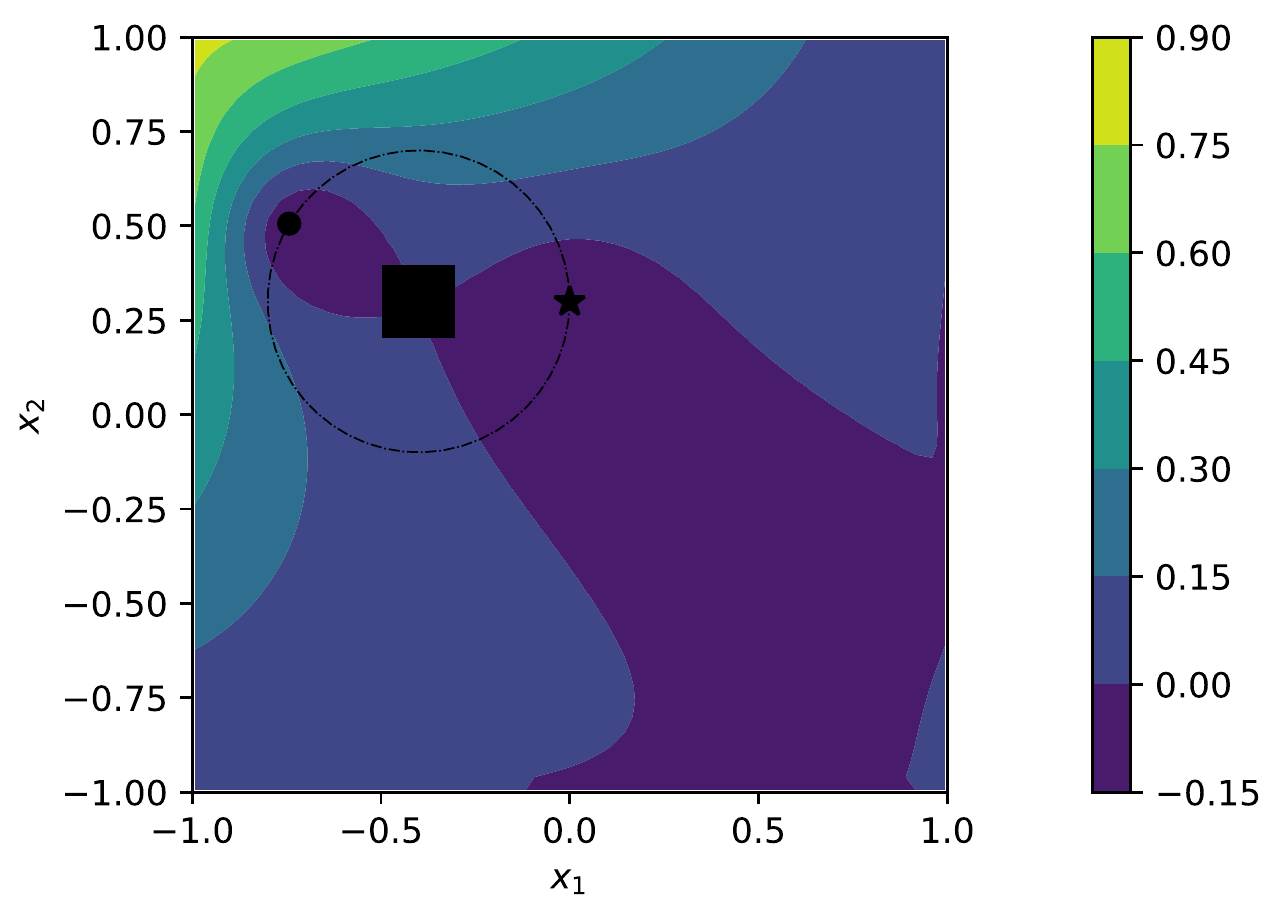}
	\caption{$t=0.4$}
	\end{subfigure}
	\begin{subfigure}[b]{0.31\textwidth}
		\includegraphics[width=\textwidth]{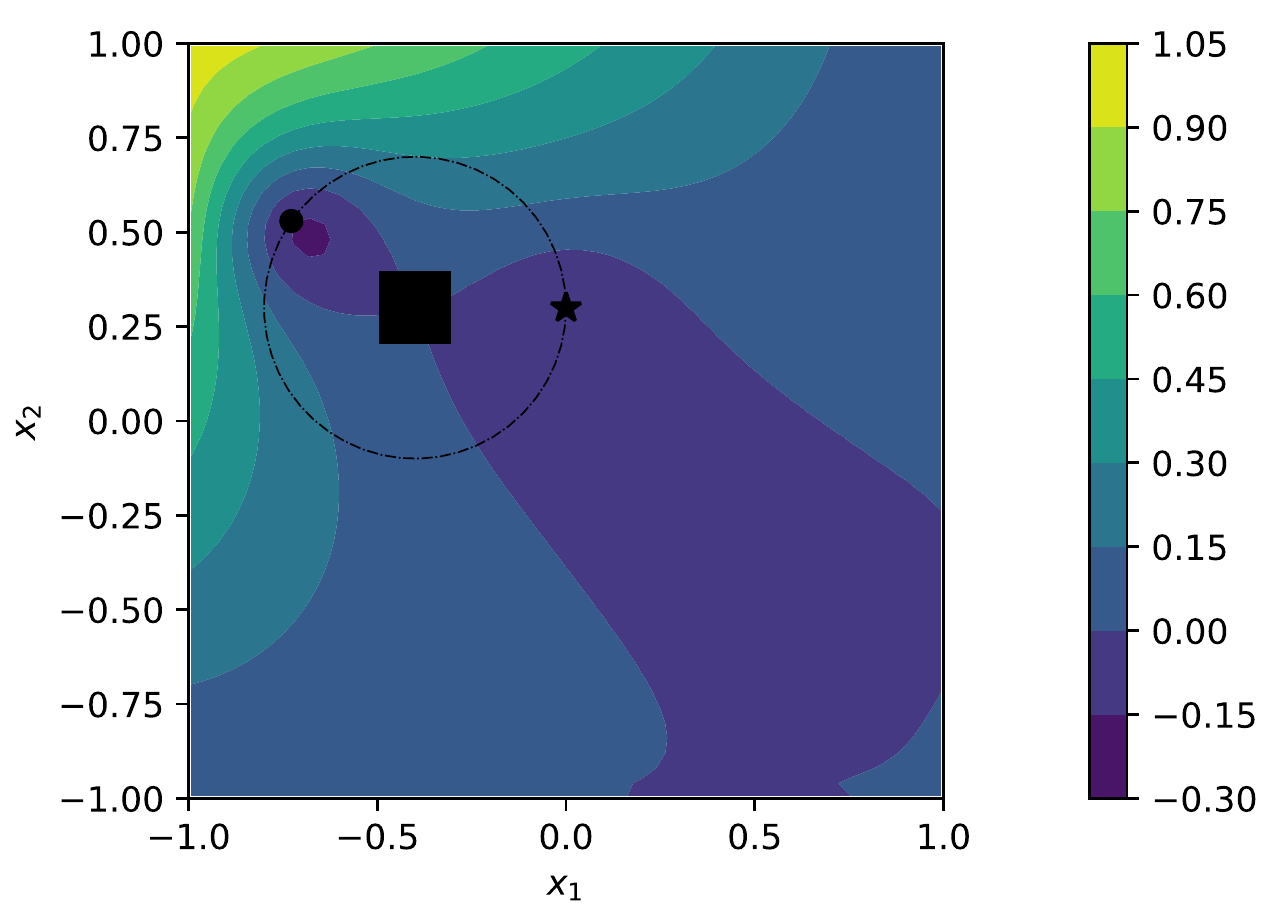}
	\caption{$t=0.6$}
	\end{subfigure}
	\begin{subfigure}[b]{0.31\textwidth}
		\includegraphics[width=\textwidth]{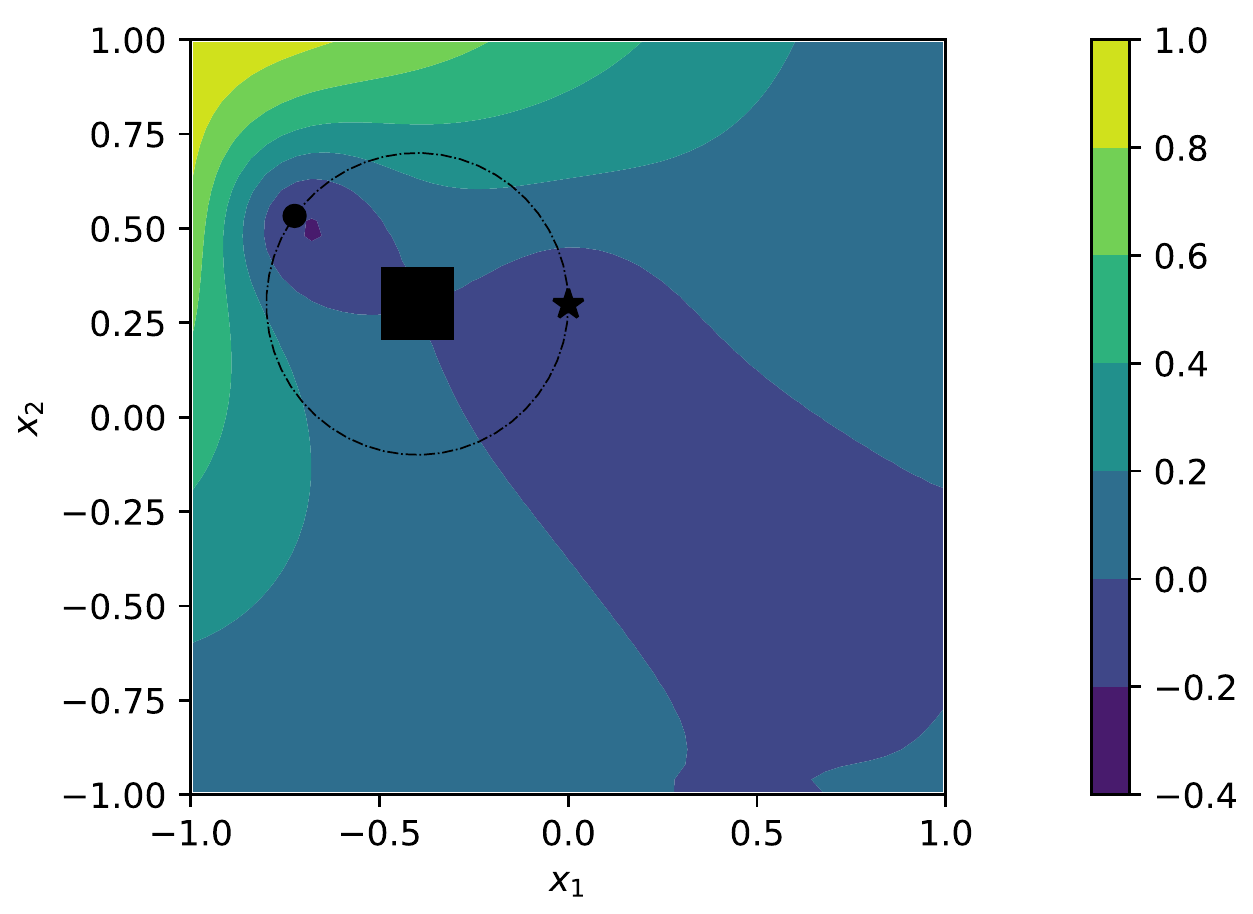}
	\caption{$t=0.8$}
	\end{subfigure}
	\begin{subfigure}[b]{0.31\textwidth}
		\includegraphics[width=\textwidth]{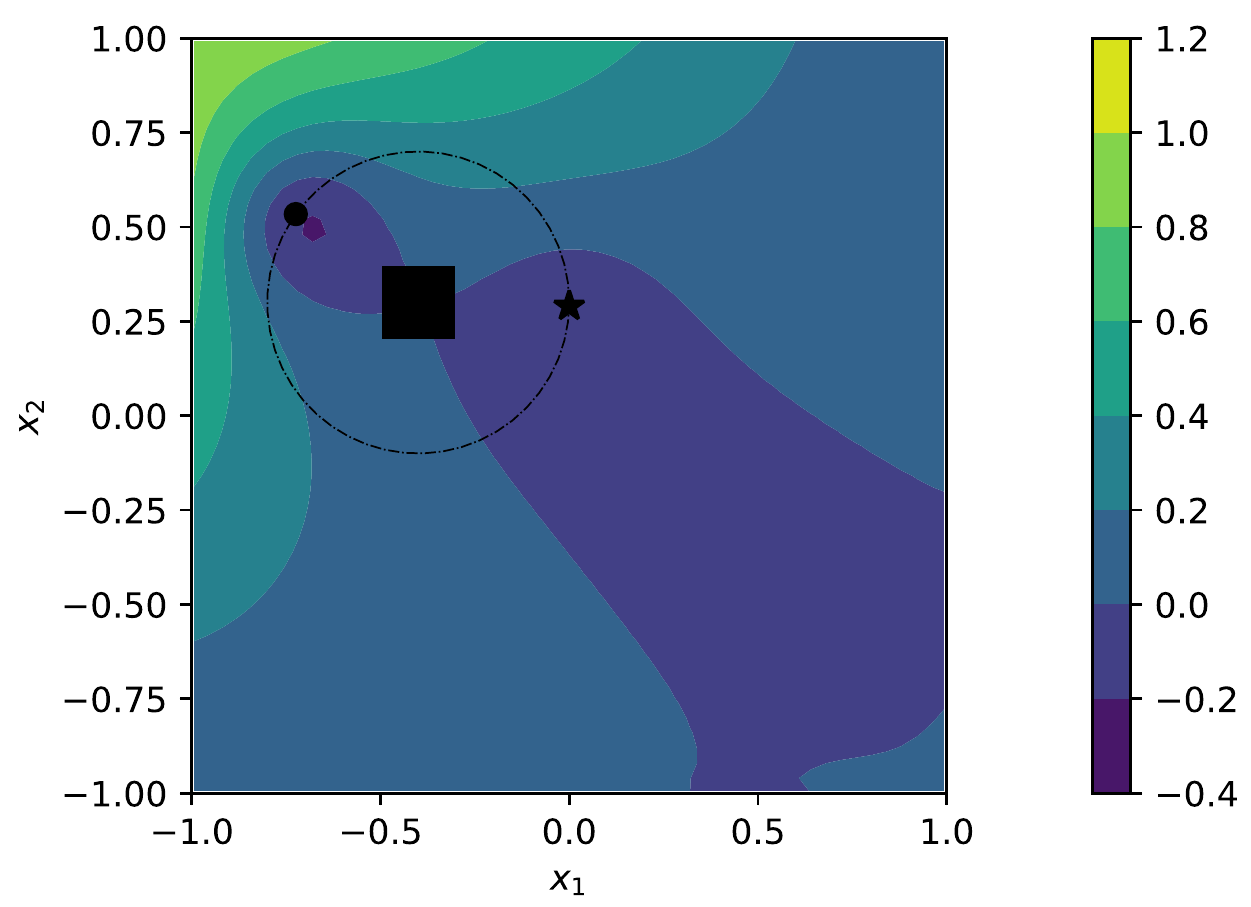}
	\caption{$t=1.0$}
\end{subfigure}
\caption{Evolution of the concentration field as the solution of the AD-PDE \eqref{eq:ADPDE} at different time instances. The black square denotes the safe zone $\Omega_d$ whereas the black circle and star denote the location of robots $1 \and 2$, respectively. Moreover, the dashed circle shows the curve $\bbgamma$. Note that in all snapshots the $c(t, \bbx)=0$ level set passes through $\Omega_d$.}	\label{fig:solPDE}
\end{figure*}
 Note that the maximum concentration that occurs at corner $\bbx_c$, grows up to $t=0.5T$. After this time, the concentration field evolves toward steady-state and the variations become smaller. As a result, the designed optimal controls also reach a steady state after $t=0.5T$, as is evident from Figure \ref{fig:sourceParam}.
An animation of the optimal source field and the corresponding concentration field is given in \cite{NNPDEanim}. Also an animation comparing the NN and FE solutions can be found in \cite{MTCerrAnim}. Note that the NN solution $f(t,\bbx)$ obtained by softly enforcing the PDE constraint in \eqref{eq:mission}, closely matches the FE solution $\hhatf(t, \bbx)$ for the final optimal source field. Specifically, the L$^2$-error between the solutions is $\left( \int_0^T \int_{\Omega} [f(t,\bbx) - \hhatf(t, \bbx)]^2 d\bbx \, dt \right)^{0.5} = 1.02 \times 10^{-02}.$

Table \ref{table:objective} presents a comparative study of the performance of the NN as a function of the number of trainable parameters $n$ and the number of training points. Particularly, it summarizes the individual terms in the objective function \eqref{eq:obj2}. The first two cases correspond to a MLP with two hidden layers with $10 \and 20$ neurons whereas the last NN, used in the previous simulations, has three hidden layers with $10, 20, \and 30$ neurons.
\begin{table*}
\centering
\renewcommand{\arraystretch}{1.3}
\captionsetup{justification=centering}
\caption{Number of training points and averaged objective values per training point for the terms in objective \eqref{eq:obj2}.}
\begin{tabular}{|c|c||c|c|c|c|} 
 \hline
 trainable parameters			& objective term 			& concentration error 		& variational form 				& IC					& BCs					\\ 	[0.5ex] 
 \hline\hline
 \multirow{2}{*}{$n = 281$}		& training points		& $n_d = 25 \times 16$			& $n_v = 25 \times 20 \times 20$	& $n_0 = 20 \times 20$	& $n_{b,i} = 25 \times 20$		\\ 	 \cline{2-6}
 							& objective value			& $3.51 \times 10^{-4}$		& $2.65 \times 10^{-5}$			& $7.78 \times 10^{-6}$	& $9.12 \times 10^{-5}$		\\ 	 \hline \hline
 \multirow{2}{*}{$n = 281$}		& training points		& $n_d = 50 \times 16$			& $n_v = 50 \times 40 \times 40$	& $n_0 = 40 \times 40$	& $n_{b,i} = 50 \times 40$	\\ 	 \cline{2-6}
 							& objective value			& $7.00 \times 10^{-5}$		& $2.06 \times 10^{-5}$			& $6.83 \times 10^{-6}$	& $5.44 \times 10^{-5}$		\\ 	 \hline \hline
 \multirow{2}{*}{$n = 921$}		& training points		& $n_d = 50 \times 16$			& $n_v = 50 \times 40 \times 40$	& $n_0 = 40 \times 40$	& $n_{b,i} = 50 \times 40$	\\ 	 \cline{2-6}
 							& objective value			& $4.28 \times 10^{-5}$		& $2.10 \times 10^{-6}$			& $1.29 \times 10^{-6}$	& $6.76 \times 10^{-6}$		\\ 	 \hline
\end{tabular}
\label{table:objective}
\end{table*}
The third column in this table reports the deviation of the predicted concentration from the desired value $c_d = 0$ in the safe zone, given by objective \eqref{eq:obj}.
Comparing the first two cases, we observe that for a fixed architecture and number of trainable parameters $n$, the normalized objective values get smaller as the number of training points increases. Moreover, comparing the last two cases, it can be seen that as the number of trainable parameters $n$ increases, the capacity of the NN increases, which amounts to smaller objective values.
%


Note that based on the above simulations, a simple feed-forward MLP-NN with only $n = 921$ trainable parameters continuously captures the $3$D concentration field across space-time. Solving the same problem using the FE method would require a mesh with $n_v = 50 \times 40 \times 40 = 8 \times 10^4$ nodes and a separate interpolation for points outside the mesh.
Note also that the proposed NN solution to the MTC problem can easily account for additional objectives and constraints, such as collision avoidance, as we already demonstrated, temporally and spatially varying concentrations in the safe zone, minimization of control effort, e.g., traveled distance or release amount, parametric PDE input-data, and desired ranges, rather than exact values, of concentration in the safe zone.
For instance, the integrand in objective \eqref{eq:obj} can be replaced with
$$ \bbarw_1 \max \left[ c(t, \bbx) - c_d(t, \bbx), 0 \right]^2 + \bbarw_2 \max \left[ c_d(t, \bbx) - c(t, \bbx) , 0 \right]^2, $$
where $\bbarw_1 \and \bbarw_2$ are weights separately penalizing concentrations above and below the desired value, respectively. If concentrations below $c_d$ are desirable, we set $\bbarw_2=0$. Note that by setting $\bbarw_1=\bbarw_2=1$, we recover the least-square objective \eqref{eq:obj}.
%

\section{Conclusion} \label{sec:concl}
We considered the MTC problem using teams of mobile robots that carry sources that counteract the release of a chemical in the environment. The goal of the robots is to steer the mass flux around a desired region, also called a safe zone, so that it remains unaffected by the external concentration.
We utilized the AD-PDE to mathematically model the transport of a chemical in the domain and formulated the problem of planning the paths and release rates of the robots as a PDE-constrained optimization. We proposed a new approach to this problem using a NN to approximate the solution of the AD-PDE. Particularly, we defined a novel loss function based on the variational form of the PDE that facilitates the training process by lowering the differentiation order and using the integral form of the PDE as opposed to its differential form. Given this loss function, we reformulated the planning problem as an unsupervised model-based learning problem. We presented simulation results that demonstrate the ability of our method to solve the desired MTC problem.



\ifCLASSOPTIONcaptionsoff
  \newpage
\fi

\bibliographystyle{ieeetr}\bibliography{MyBibliography}

\end{document}